\RequirePackage[svgnames]{xcolor}

\documentclass[11pt,letterpaper]{mystyle}

\usepackage[all]{hypcap}
\usepackage[svgnames]{xcolor}
\usepackage[comma,authoryear,compress]{natbib}
\usepackage{url}            
\usepackage{booktabs}       
\usepackage{amsfonts}       
\usepackage{microtype}      
\usepackage{lipsum}         
\usepackage{amsthm}
\newtheorem{definition}{Definition}
\newtheorem{proposition}{Proposition}

\usepackage{graphicx}  
\usepackage{float} 
\usepackage{algorithm}
\usepackage{algorithmic}
\usepackage{tabulary}
\usepackage{xspace}

\usepackage{amsmath}
\usepackage{amssymb}
\usepackage[table,dvipsnames]{xcolor}
\usepackage{comment}
\usepackage{multirow}
\usepackage{booktabs}
\usepackage{makecell}
\usepackage{subcaption}
\usepackage{fancybox}
\usepackage{fancyvrb}
\usepackage{multirow}
\usepackage{tikz}                  %
\usepackage[edges]{forest}
\usepackage{colortbl}
\usepackage{framed}
\usepackage{tcolorbox}
\usepackage{wrapfig}

\usepackage{pgfplots}
\pgfplotsset{compat=1.18}

\usepackage{enumitem}
\setlist[itemize]{topsep=1pt, itemsep=1pt}

\usepackage{url}            
\usepackage{booktabs}       
\usepackage{multirow}    
\usepackage{amsfonts}       
\usepackage{nicefrac}       
\usepackage{microtype}      
\usepackage{natbib}
\usepackage{enumerate}
\usepackage{hhline}
\usepackage{makecell}
\usepackage{pifont}

\usepackage{graphicx} 
\usepackage{amsmath}
\usepackage{amsthm}
\usepackage{amssymb}
\usepackage{tikz}
\usepackage{xcolor}
\usetikzlibrary{arrows}

\allowdisplaybreaks

\usepackage{mathrsfs}

\usepackage{hyperref}
\usepackage{bm}

\allowdisplaybreaks

\newcommand{\expect}{\mathbb{E}}

\usepackage[capitalize,noabbrev]{cleveref}
\crefname{thm}{Theorem}{Theorems}
\crefname{lem}{Lemma}{Lemmas}
\crefname{cor}{Corollary}{Corollaries}
\crefname{prop}{Proposition}{Propositions}
\crefname{asmp}{Assumption}{Assumptions}
\crefname{defn}{Definition}{Definitions}
\crefname{oracle}{Oracle}{Oracles}
\crefname{fact}{Fact}{Facts}
\crefname{conj}{Conjecture}{Conjectures}
\crefname{rem}{Remark}{Remarks}
\crefname{example}{Example}{Examples}
\crefname{condition}{Condition}{Conditions}
\crefname{exercise}{Exercise}{Exercises}
\crefname{algorithm}{Algorithm}{Algorithms}
\crefname{table}{Table}{Tables}
\crefname{figure}{Figure}{Figures}
\crefname{section}{Section}{Sections}
\crefname{subsection}{Section}{Sections}
\crefname{appendix}{Appendix}{Appendices}
\crefname{message}{Message}{Messages}

\definecolor{red}{rgb}{1, 0, 0}
\newcommand{\RED}[1]{{\color{red} #1}}

\definecolor{green}{rgb}{0, 1, 0}

\definecolor{blue}{rgb}{0, 0, 1}

\definecolor{orange}{rgb}{1, 0.4, 0.0}

\input{packages/math_commands}

\usepackage[utf8]{inputenc} %
\usepackage[T1]{fontenc}    %
\usepackage{nicefrac}       %

\renewcommand*{\backrefalt}[4]{%
    \ifcase #1 \footnotesize{(Not cited.)}%
    \or        \footnotesize{(Cited on page~#2.)}%
    \else      \footnotesize{(Cited on pages~#2.)}%
    \fi}

\tcbset{
  aibox/.style={
    width=\linewidth,
    top=7pt,
    bottom=2pt,
    colback=blue!6!white,
    colframe=black,
    colbacktitle=black,
    enhanced,
    center,
    attach boxed title to top left={yshift=-0.1in,xshift=0.15in},
    boxed title style={boxrule=0pt,colframe=white,},
  }
}
\newtcolorbox{AIbox}[2][]{aibox,title=#2,#1}

\definecolor{lightblue}{rgb}{0.22,0.45,0.70}%
\definecolor{Gray}{gray}{0.95}
\definecolor{Cornsilk}{rgb}{1.0, 0.97, 0.86}

\usepackage{amsmath}

\usepackage[all]{hypcap}

\usepackage{listings}
\hypersetup{
  colorlinks   = true, %
  urlcolor     = {TinaCrimson}, %
  linkcolor    = {TinaCrimson}, %
  citecolor   = {YaleBlue} %
}

\newcommand{\piref}{\pi_{\text{ref}}}
\newcommand{\thetaref}{\theta_{\text{ref}}}

\newcommand{\methodfull}{Off-Policy Value-Based Reinforcement Learning with Replay Buffer\xspace}
\newcommand{\method}{ReVal\xspace}

\newcommand{\ttt}[1]{\texttt{#1}}

\title{Off-Policy Value-Based Reinforcement Learning for Large Language Models}

\usepackage{authblk}

\author[1,*]{Peng-Yuan Wang}
\author[2,3,*]{Ziniu Li}
\author[1,*]{Tian Xu}
\author[1]{Bohan Yang}
\author[1]{Tian-Shuo Liu}
\author[1]{ChenYang Wang}
\author[1]{Xiong-Hui Chen}
\author[1]{Yi-Chen Li}
\author[3]{Tianyun Yang}
\author[3,4]{Congliang Chen}
\author[1,$\dagger$]{Yang Yu}

\affil[1]{National Key Laboratory for Novel Software Technology \& School of Artificial Intelligence, Nanjing University, China}
\affil[2]{The Chinese University of Hong Kong, Shenzhen}
\affil[3]{Shenzhen Research Institute of Big Data}
\affil[4]{Shenzhen Loop Area Institute, Shenzhen}

\begin{document}

\begin{abstract}
\textbf{Abstract:} Improving data utilization efficiency is critical for scaling reinforcement learning (RL) for long-horizon tasks where generating trajectories is expensive. However, the dominant RL methods for LLMs are largely \emph{on-policy}: they update each batch of data only once, discard it, and then collect fresh samples, resulting in poor sample efficiency. In this work, we explore an alternative \emph{value-based} RL framework for LLMs that naturally enables \emph{off-policy} learning. We propose \method, a Bellman-update-based method that combines stepwise signals capturing internal consistency with trajectory-level signals derived from outcome verification. \method naturally supports replay-buffer-based training, allowing efficient reuse of past trajectories. Experiments on standard mathematical reasoning benchmarks show that \method not only converges faster but also outperforms GRPO in final performance. On DeepSeek-R1-Distill-1.5B, \method improves training efficiency and achieves improvement of \textbf{2.7\%} in AIME24 and \textbf{4.5\%} in out-of-domain benchmark GPQA over GRPO. These results suggest that value-based RL is a practical alternative to policy-based methods for LLM training.
\end{abstract}

\maketitle

\vspace{3mm}
\begingroup
  \renewcommand{\thefootnote}{*}
  \footnotetext{Equal contribution.}
\endgroup

\begingroup
  \renewcommand{\thefootnote}{$\dagger$}
  \footnotetext{Corresponding author. Email: yuy@nju.edu.cn}
\endgroup

\section{Introduction}

Since the advent of reinforcement learning from human feedback (RLHF), reinforcement learning (RL) has become a central component of large language model (LLM) post-training~\citep{guo2025deepseek, team2025kimi, li2025review, wang2026survey}. In particular, reinforcement learning with verifiable rewards (RLVR) has proven highly effective for improving the reasoning ability of LLMs by training them from correctness signals on complete responses~\citep{lambert2024tulu, o1, guo2025deepseek}.

Because autoregressive LLMs are naturally parameterized as token-level policies, actor-critic policy optimization algorithms such as PPO~\citep{schulman2017ppo} initially became the dominant approach for RL-based post-training. They offered a stable and conceptually straightforward framework for optimizing pretrained language models. As the field matured, however, it became clear that at LLM scale, an RL algorithm is only practical if it is computationally efficient. ReMax~\citep{li2024remax} was the first to move from actor-critic to actor-only RL, significantly reducing memory usage and training time for LLM post-training. Following ReMax, a series of methods, including GRPO~\citep{shao2024deepseekmath} and DAPO~\citep{qiying2025dapo}, further advanced this low-cost policy optimization paradigm.

However, these actor-only methods remain fundamentally \emph{on-policy}. Updates must be computed from data sampled from the current policy, so collected trajectories quickly become stale and can only be reused to a limited extent. For short-horizon tasks, this inefficiency may be acceptable. But LLM development is increasingly shifting toward \emph{agentic} settings with long and highly variable horizons, where trajectory collection is expensive and often dominates the total training cost~\citep{gao2025rollpacker, team2026kimi}. In this regime, reducing per-update overhead is no longer sufficient. The next fundamental requirement is the ability to \emph{reuse experience}, that is, RL for LLMs must become off-policy.

In classical RL, this naturally points to value-based methods, whose efficiency comes from Bellman learning and whose replay-buffer-based training readily supports off-policy data reuse~\citep{watkins1992q,dqn_atari,mnih2015human}. Yet despite these advantages, standard value-based formulations are not directly compatible with LLM post-training. Conventional value-based RL typically relies on a value model that explicitly predicts values, which is apparently incompatible with LLMs. More importantly, introducing such an additional value model would undermine the very low-cost property that makes actor-only methods attractive at LLM scale by increasing both memory and computation overhead. 

A recent work by \cite{li2025generalist} offers a way around this obstacle. It shows that the logits of a pretrained LLM can be interpreted as parameterizing action values of an endogenous reward, up to a state-dependent transformation. If logits can serve as \(Q\)-values, then policy and value no longer need to be represented by separate models. Once a pretrained LLM is viewed as an endogenous value model, a single-model, low-cost, off-policy RL algorithm becomes possible. This perspective leads directly to our method.

\begin{figure}[b!]
    \centering
    \includegraphics[width=\linewidth]{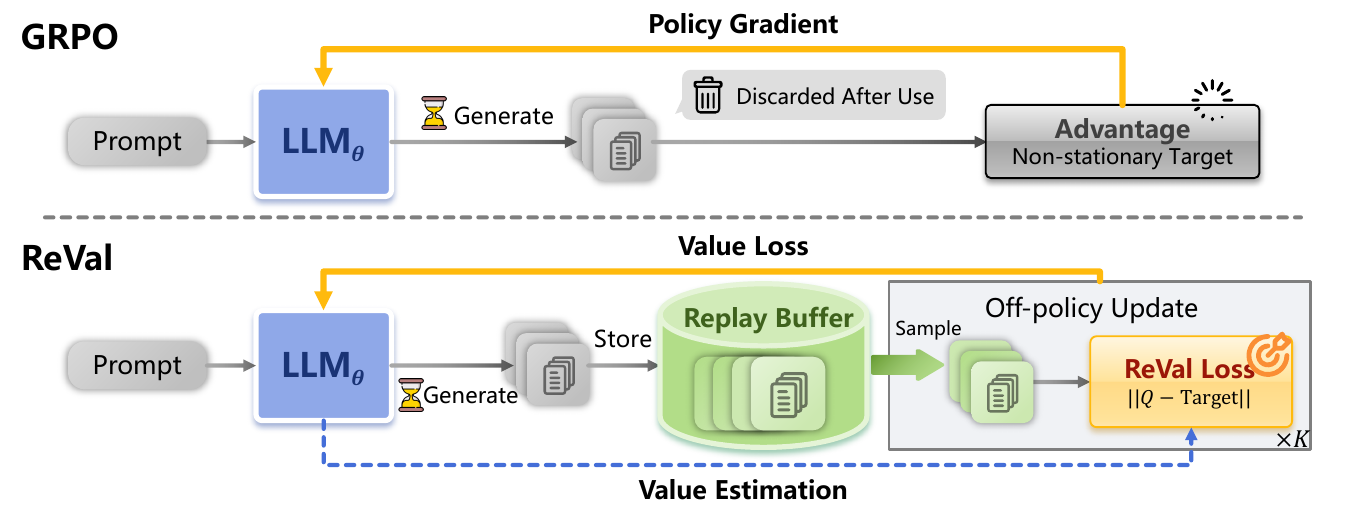}
    \caption{Framework of \method. By interpreting LLM logits as \(Q\)-values, \method unifies policy and value within a single model and enables replay-based off-policy updates.}
    \label{fig:framework}
\end{figure}

In this paper, we propose \method, a value-based RL framework for LLM post-training that preserves the efficiency advantages of ReMax while introducing the off-policy capability required for agentic and long-horizon learning. Our method is built on two key principles. First, on the \emph{objective} side, effective value learning for LLMs should combine supervision at different temporal scales: stepwise signals provide dense feedback by encouraging internal consistency, while trajectory-level signals convey outcome-level correctness from verification. Either signal alone is insufficient: stepwise feedback by itself is inefficient at reflecting trajectory-level outcomes, whereas trajectory-level Bellman learning alone, as in prior work~\citep{yuan2025trajectory}, can suffer from mis-calibration at initialization. We therefore introduce a reward-shaping formulation tailored to logit-parameterized \(Q\)-functions that naturally integrates both stepwise and trajectory-level signals, yielding substantially more stable optimization. Second, on the \emph{data} side, value-based RL should fully exploit replay. We therefore introduce a replay-buffer training mechanism that repeatedly reuses historical trajectories, converting what would otherwise be discarded rollouts into useful supervision. Together, these designs yield a practical single-model value-learning algorithm that is both computationally efficient and genuinely off-policy.

Empirically, we first show that increased off-policy reuse directly accelerates learning. With more frequent replay updates, \method reaches comparable performance substantially faster, achieving an average \(4.3\times\) speedup over GRPO. We then evaluate \method on standard mathematical reasoning benchmarks and find that it consistently outperforms strong policy-based baselines in both convergence speed and final accuracy. On DeepSeek-R1-Distill-1.5B, \method improves over GRPO by \textbf{2.7\%} on AIME24 and \textbf{4.5\%} on the out-of-domain benchmark GPQA. On Qwen2.5-Math-7B, it further surpasses GRPO by \textbf{4.3\%} on GPQA. The advantage is even more pronounced in the limited-rollout setting \((N=1)\), where fresh trajectories are scarce and off-policy reuse is especially valuable: in this setting, \method exceeds GRPO by \textbf{4.8\%} on AIME and \textbf{4.6\%} on GPQA. We further provide ablations on KL regularization, the hyperparameter \(\beta\), and reward design, yielding practical guidance for stable value-based RL in LLMs. Overall, \method shows that it is possible to make RL \emph{off-policy} without making it more expensive, by unifying value and policy within the pretrained LLM itself.

\section{Preliminaries}

\subsection{LLM and its MDP Formulation}

\paragraph{Basic Introduction on LLM.} A large language model (LLM) is a generative model that predicts the next token in a sequence using probabilistic modeling. Formally, an LLM $\pi$ generates tokens from a finite vocabulary $\gV = \{1,2, \ldots, |\gV| \}$ and generates a sequence in an autoregressive manner. At step $h$, given a context sequence $(a_1, \ldots, a_{h-1})$, an LLM produces the next token according to the conditional distribution, namely, $a_h \sim \pi (\cdot|a_1, \ldots, a_{h-1})$. This process continues until a designated end-of-sequence (EOS) token is generated or a predefined maximum length $H$ is reached. For analytical clarity, we assume uniform response lengths of exactly $H$, with padding applied after the EOS token as needed.

\paragraph{MDP Formulation of LLM.} We adopt the Markov decision process (MDP) formulation of LLMs from \citep{li2024remax}, defined by the tuple $\mathcal{M} = \langle \mathcal{S}, \gV, r, P, \rho, H\rangle$. The state space $\gS$ is the set of all finite-length strings formed by the concatenation of elements in $\gV$ and the action space is the vocabulary set $\gV$. When generating a response, the initial state (prompt) $s_1 = (x_1, x_2, \cdots, x_m)$ is sampled from the initial state distribution $\rho$, with $m \in \mathbb{N}$ and $\forall i \in [m], x_i \in \mathcal{V}$. At each step $h \in [H]$, the LLM selects an action (or equivalently, a token) $a_h \in \mathcal{V}$ according to $\pi(\cdot|s_h)$. 
The environment then transits to the next state $s_{h+1} = (x, a_1, \cdots, a_h)$, rewarding the LLM with $r(s_h, a_h) \in [0, 1]$. That is, the transition model $P: \mathcal{S}\times\mathcal{V}\to \Delta (\mathcal{S})$ is usually deterministic. $ P(s_{h+1}|s_h,a_h)=1$ if and only if $s_{h+1} = s_h \oplus a_h$, where $\oplus$ means concatenation. 
The trajectory ends after a total of $H$ steps. In the context of RL, we also call $\pi$ as a policy. Throughout this paper, the terms ``policy'' and ``LLM'' will be used interchangeably.

\subsection{Reinforcement Learning with Verifiable Reward} Reinforcement Learning with Verifiable Reward (RLVR) has become a widely adopted paradigm in LLM reasoning following recent breakthroughs such as OpenAI-o1~\citep{o1} and DeepSeek-R1~\citep{guo2025deepseek}.
Unlike RLHF, which relies on learned reward models, RLVR trains LLMs by maximizing a \emph{rule-based outcome reward} with KL-regularization:
\begin{align} 
     \max_{\theta} \;
\mathbb{E}_{x \sim \rho} \Big[
    \mathbb{E}_{a_{1:H} \sim \pi_{\theta} (\cdot \mid x)}
    \big[
        r_{\text{rule}} (x, a_{1:H})
    \big] - \beta \,
\KL\!\big( 
    \pi_{\theta}(\cdot \mid x),
    \pi_{\text{ref}}(\cdot \mid x)
\big) \Big] .
     \label{eq:reward_maximization}
\end{align}
This rule-based reward $r_{\text{rule}}$ evaluates the correctness of the final answer based on deterministic verification procedures. For mathematical tasks, one can directly compare the final answer in the response with the ground-truth answer, checking for mathematical equivalence. Besides, $\KL (\pi_{\theta}(\cdot | x),
    \pi_{\text{ref}}(\cdot | x)) = \sum_{a_{1:H}} \pi_{\theta} (a_{1:H}|x) \log (\pi_{\theta} (a_{1:H}|x) / \pi_{\text{ref}} (a_{1:H}|x) ) $ denotes the KL divergence, which prevents the learning model from deviating too far from the reference model and $\beta > 0$ controls the regularization strength.

\section{Limitations of On-Policy Methods}
\label{sec:limitation_of_off}

\begin{wrapfigure}{r}{0.4\linewidth}
\centering
\includegraphics[width=\linewidth]{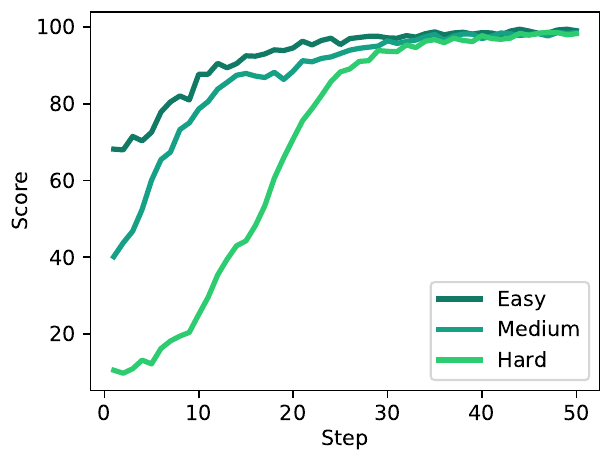}
\caption{Performance of GRPO across different difficulty levels.}
\label{fig:grpo_by_difficulty}
\end{wrapfigure}
We revisit the foundations of policy gradient methods and examine their implications for practical training cost. By construction, policy gradient methods \citep{sutton1988learning} update the policy at iteration $k$ using gradient estimates computed from trajectories sampled under the current policy $\pi_k$. \emph{After the update, those trajectories are no longer on-policy and therefore cannot be reused directly in subsequent iterations; fresh samples must be collected again.} Because each policy update is typically small and local, a single gradient step cannot fully exploit the information contained in one batch of trajectories~\citep{bottou2010large}. As a result, convergence generally requires many iterations of alternating data collection and optimization.

We empirically demonstrate this limitation using a one-shot task learning setting, where training is conducted on a single prompt (Figure~\ref{fig:grpo_by_difficulty}). GRPO~\citep{shao2024deepseekmath}, as an on-policy method, collects fresh trajectories at every iteration and therefore directly reflects the sampling difficulty of the task. On hard task (\texttt{avg@1024} = 0.10), GRPO requires substantially more optimization steps than on medium (\texttt{avg@1024} = 0.40) or easy (\texttt{avg@1024} = 0.68) task to reach the same performance threshold. These findings indicate that: (1) even in a highly simplified one-shot setting, learning still depends on repeated alternation between data collection and parameter updates; and (2) harder tasks may require many more such iterations before they can be solved.


These inefficiencies translate directly into practical bottlenecks. Let $K_{\mathrm{generation}}$ denote the number of generation rounds and $K_{\mathrm{update}}$ the number of parameter updates. In standard on-policy training, the two are tightly coupled, so typically $K_{\mathrm{generation}} = K_{\mathrm{update}}$, since each update requires newly sampled trajectories. If the time cost of one generation round is $T_{\mathrm{generation}}$ and that of one update is $T_{\mathrm{update}}$, then the total training time can be approximated as
\[
T_{\mathrm{total}} \approx K_{\mathrm{generation}} T_{\mathrm{generation}} + K_{\mathrm{update}} T_{\mathrm{update}}.
\]
For RL with LLMs, $T_{\mathrm{generation}}$ is often much larger than $T_{\mathrm{update}}$ because of the cost of autoregressive sequence generation~\citep{qin2025seer}. Consequently, reducing total training time requires either decreasing $T_{\mathrm{generation}}$ itself, for example through faster generation systems \citep{leviathan2022fast}, or decreasing $K_{\mathrm{generation}}$, that is, learning more efficiently from each round of collected data \citep{yu2018towards}. In other words, the central challenge is to extract more useful learning signal from the same set of trajectories.

This observation motivates value-based and off-policy methods, which decouple data collection from policy updates and allow historical trajectories to be reused across multiple optimization steps \citep{watkins1992q, mnih2015human, td3}. By performing more parameter updates on the same batch of trajectories, the algorithm can extract more learning signal from each generation round. This improved reuse can accelerate convergence, because the model makes greater progress before new data need to be collected. As a result, the total number of generation rounds required during training may decrease. Although this strategy introduces additional update cost, parameter updates are cheaper than autoregressive generation in LLM-based RL. Therefore, when we reduce generation rounds, the overall wall-clock training time can also decrease. From this perspective, the benefit of off-policy methods lies not only in improved sample efficiency, but also in lower practical training cost through more effective reuse of generated trajectories. In the next section, we explore value-based and off-policy methods from this perspective.

\section{Proposed Method}

\subsection{Towards Value-Based RL for LLMs}
\paragraph{Q-Function Parameterization in LLMs.}
A key challenge in applying value-based RL to LLMs is how to represent or initialize the Q-function. In standard Q-learning, the Q-function is parameterized as a mapping from a state to a vector of Q-values over all actions \citep{watkins1992q, mnih2015human}, i.e., $f(s_h) \rightarrow \mathbb{R}^{|\mathcal{A}|}$, enabling greedy action selection by taking the $\arg\max$ over the output. Unlike standard RL settings where such a Q-function can be learned from a randomly initialized network, this approach is infeasible for LLMs for two reasons. First, the token vocabulary constitutes an enormous action space and rewards are typically sparse, making it difficult to learn a reliable Q-function from outcome signals alone. Second, learning a Q-function from scratch requires a large amount of data, whereas the amount of data available in RL fine-tuning is far from sufficient.


\citet{li2025generalist} established a principled solution to this challenge. They showed that a language model trained via next-token prediction implicitly learns a soft Q-function: given a language model $\hat{\pi}$ parameterized as $\hat{\pi}(\cdot \mid s_h) = \text{softmax}(\hat{f}(s_h, \cdot))$, the logits $\hat{f}(s_h, a_h)$ directly correspond to the soft Q-values of the data-generating policy. Q-values can be learned from the implicit rewards in pretraining data through an inverse RL formulation. This reveals that LLM logits are not arbitrary scores but encode value-relevant information about token-level decisions, providing a well-initialized Q-function for free.

\label{subsec:method}

\paragraph{TBRM.}
Similar to \citet{li2025generalist}, TBRM~\citep{yuan2025trajectory} adopts the logit-as-$Q$ parameterization:
\begin{align*}
    Q_{\theta}(s_h, a_h) := \text{logit}_{\theta}(s_h, a_h),
\end{align*}
where the LLM's own logits serve as the Q-function without requiring a separate Q-value network.
Given this Q-function parameterization, TBRM learns by minimizing the trajectory-level Bellman residual in KL-regularized RL framework. 
TBRM minimizes the trajectory-level Bellman residual:
\begin{align*}
    \mathcal{L}_{\text{TBRM}}(\theta) 
    &= \frac{1}{|\hat{\mathcal{D}}|} \sum_{\tau \in \hat{\mathcal{D}}} \left( \log \pi_{\theta}(\tau) - \log \pi_{\text{ref}}(\tau) - \frac{r_{\text{rule}}(\tau)}{\beta} + V_{{\theta}}(s_1) \right)^2,
\end{align*}
where $\hat{\mathcal{D}}$ denotes the on-policy data which is collected from current policy. $\pi(\tau) = \prod_{h=1}^H \pi(a_h|s_h)$ denotes the probability of trajectory $\tau$ and $V_{\theta}(s_1)$ is the induced V-function, $V_{\theta}(s_1)=\log \sum_{a \in \mathcal{A}} \exp Q(s_{1}, a)$.

TBRM showed empirical success with on-policy data. However, we have identified that TBRM does not satisfy \emph{Calibrated Initialization}, which leads to spurious policy drift in the absence of reward signals. In the next section, we propose our training objective to address these limitations.
\subsection{\methodfull (\method)}
We begin by examining the limitations of the TBRM training objective. First, we show that a desirable property for the training objective is defined as follows.

\begin{definition}[Calibrated Initialization]
A training objective satisfies \emph{Calibrated Initialization} if, when $r_\text{rule} = 0$, the optimal policy under the KL-regularized RL objective reduces to the reference policy, i.e., $\pi^* = \pi_{\text{ref}}$.
\end{definition}

At the beginning of training, when no reward signal is available (i.e., $r_\text{rule} = 0$), the desired behavior is to leave the policy unchanged, i.e., $\pi^* = \pi_\text{ref}$. If this property is not satisfied, the model will still produce parameter updates, leading to spurious policy drift. 

\begin{proposition}
TBRM does not satisfy \emph{Calibrated Initialization}. Specifically, setting $r(\tau) = 0$ in the TBRM objective does not yield $\pi^* = \pi_\text{ref}$ as the optimal solution.
\label{prop:tbrm}
\end{proposition}

However, TBRM does not satisfy this property, as stated in Proposition~\ref{prop:tbrm}. Setting $r(\tau)=0$ in the TBRM objective yields:
\[
\mathcal{L}_{\text{TBRM}}(\theta) = \left(V_\theta(s_1) + \sum_{h=1}^{H} \log \frac{\pi_\theta(a_h \mid s_h)}{\pi_{\text{ref}}(a_h \mid s_h)}\right)^2.
\]
Minimizing this drives the squared term toward zero, which requires the log-likelihood ratio to cancel $V_\theta(s_1)$ rather than merely minimizing the KL divergence between $\pi_\theta$ and $\pi_{\text{ref}}$. As a result, the optimal solution does not correspond to matching the reference policy. We provide a detailed discussion and empirical verification in Appendix~\ref{sec:issue_tbrm}. 

To address this issue, we introduce reward shaping to redefine the Bellman objective. Specifically, we define a \emph{modified reward} function:
\begin{align*}
R_\beta(s_h, a_h) := \frac{r_{\text{rule}}(s_h, a_h)}{\beta} + \log \pi_{\text{ref}}(a_h \mid s_h) + \underbrace{\textcolor{red}{V_\theta(s_h) - V_{\text{ref}}(s_h)}}_{\text{reward shaping term}},
\end{align*}
where the first term is the scaled environment reward, and the remaining terms form an endogenous reward~\citep{li2025generalist} guided by the reference policy. Intuitively, the endogenous reward incorporates the reference policy to guide the model, and the reward shaping term introduces a state-dependent offset that does not affect the optimal solution~\citep{ng1999policy}.

Based on this modified reward, we define the Bellman operator as:
\begin{align*}
(\mathcal{T}_{\beta} Q)(s_h, a_h)
&= \underbrace{\frac{r_{\text{rule}}(s_h, a_h)}{\beta}}_{\text{task reward}}
+ \log \pi_{\text{ref}}(a_h \mid s_h)
+ \underbrace{V_{\text{ref}}(s_h) - V_{\text{ref}}(s_{h+1})}_{\text{reward shaping term}} \\
&\quad + \mathbb{E}_{s_{h+1} \sim \mathcal{P}(\cdot \mid s_h, a_h)} \left[ \log \sum_{a \in \mathcal{A}} \exp Q(s_{h+1}, a) \right],
\end{align*}
where $V_{\text{ref}}(s_h) = \log \sum_{a \in \mathcal{A}} \exp Q_{\text{ref}}(s_h, a)$ and $\pi_{\text{ref}}(\cdot \mid s_h) = \text{softmax}(Q_{\text{ref}}(s_h, \cdot))$. The trajectory-level Bellman residual loss is then:
\begin{align}
\mathcal{L}_{\text{ReVal}}(\theta)
&= \frac{1}{|\mathcal{D}|} \sum_{\tau \in \mathcal{D}} \left( \sum_{h=1}^H Q_\theta(s_h, a_h) - (\mathcal{T}_{\beta} Q_\theta)(s_h, a_h) \right)^2 \notag \\
&= \frac{1}{|\mathcal{D}|} \sum_{\tau \in \mathcal{D}} \left( V_\theta(s_1) - \textcolor{red}{V_{\text{ref}}(s_1)} + \log \pi_\theta(\tau) - \frac{r_{\text{rule}}(\tau)}{\beta} - \log \pi_{\text{ref}}(\tau) \right)^2.
\label{eq:loss}
\end{align}
Where $D$ denotes the off-policy data. This formulation ensures \emph{Calibrated Initialization}, as stated in the following proposition (the proof is shown in Appendix~\ref{sec:proof}). The Eq.~\ref{eq:loss} is used for optimizing as shown in Algorithm~\ref{alg:brm_mc}. 

\begin{proposition}
\method satisfies Calibrated Initialization. Consider the objective in Eq.~\ref{eq:loss}. When $r=0$ and the policy is initialized as $\pi_\theta = \pi_{\text{ref}}$, we have $V_\theta(s_1) = V_{\text{ref}}(s_1)$ and $\log \pi_\theta(\tau) = \log \pi_{\text{ref}}(\tau)$, which yields $\mathcal{L}_{\text{ReVal}}(\theta) = 0$.
\end{proposition}

\begin{algorithm}[htbp]
    \caption{Off-Policy Value-Based Reinforcement Learning with Replay Buffer (\method)}
    \label{alg:brm_mc}
    \begin{algorithmic}[1]
    \REQUIRE{Task prompt dataset $\gD_{\text{task}}$, first-in-first-out (FIFO) replay buffer $\gD_{\text{replay}} = \emptyset$, task reward $r$, reward scaling coefficient $\beta$, reference policy $\piref$ with parameter $\thetaref$, number of iterations $T$.}
    \STATE{Initialize.}
    \FOR{$t = 1, 2, \ldots, T$}
        \STATE{For each question $q \in \gD^t_{\text{task}}$, sample trajectories from policy $\pi_{\theta}$ , and collect these trajectories into batch $\gD^t$.}

        \STATE{Augment buffer with the batch $\mathcal{D}_{\text{replay}} \leftarrow \mathcal{D}_{\text{replay}} \cup \mathcal{D}^t$ and evict oldest samples if capacity is exceeded.}
        
        \STATE{Sample an off-policy batch $\gD^t_{\text{replay}} \subset \gD_{\text{replay}}$.}
        \STATE{\RED{Update $\theta$ via gradient descent using Eq.~\ref{eq:loss} based on off-policy batch $\gD^t_{\text{replay}}$.}}
        \STATE{Sample a prompt batch $\gD^t_{\text{task}} \subset \gD_{\text{task}}$.}
        
    \ENDFOR
    \end{algorithmic}
    \end{algorithm}
\subsection{Replay Buffer for Off-Policy Learning}

A key advantage of value-based RL over policy gradient methods is its natural compatibility with off-policy data. TBRM operates in an on-policy manner, discarding each batch of trajectories after a single update. \method introduces a replay buffer $\mathcal{D}_{\text{replay}}$ that stores historical trajectories and enables off-policy learning, which satisfy desirable properties of value-based RL. At each iteration, newly collected trajectories are added to the buffer, and a batch is sampled from the full buffer for the gradient update, allowing efficient reuse of past experience.

We adopt a first-in-first-out (FIFO) replay buffer of size $M$. At each iteration, $B$ new trajectories are collected and stored in the buffer. We then perform $K$ updates per iteration, each sampling a batch of size $B$ uniformly from the buffer. A trajectory remains in the buffer for $\left\lfloor M / B \right\rfloor$ iterations before being evicted. Since each trajectory is sampled with probability $B/M$ per update step, the expected total number of gradient updates a single trajectory contributes to is:
\begin{equation}
    \mathbb{E}[\text{updates per trajectory}] = \left\lfloor \frac{M}{B} \right\rfloor \cdot \frac{B}{M} \cdot K \approx K.
\end{equation}
In practice, we use $B = 1024$, $M = 5120$, and $K = 2$, yielding an expected reuse of $K \approx 2$ gradient updates per trajectory, compared to the one-shot usage in on-policy methods. We leave the exploration of more efficient sampling strategies, such as prioritized experience replay~\citep{schaul2015prioritized}, to future work.

\section{Experiments}

\subsection{Experimental Setup}
\paragraph{Experiment Setting.} We implement \method and baseline methods using the large-scale RL
training framework Verl~\citep{sheng2024hybridflow}. Our primary focus is on GRPO~\citep{shao2024deepseekmath}, which are widely examined policy optimization methods in LLM training. We also adopt a value-based baseline, namely TBRM~\citep{yuan2025trajectory}. We train our models using the DeepScaleR dataset~\citep{deepscaler2025}. All methods are trained for 650 iterations.

We conduct experiments with DeepSeek-R1-DistillQwen-1.5B~\citep{guo2025deepseek} (abbreviated as DPSK-R1-Distill-1.5B) and Qwen2.5-Math-7B~\citep{yang2024qwen25mathtechnicalreportmathematical}. In each iteration, we employ a batch size of M = 128 prompts and generate
N = 8 rollouts. All responses are sampled with a
temperature of 1.0. More details of our implementation can be found in Appendix \ref{sec:detailed_exp}. For evaluation, we follow \citet{deepscaler2025} and assess our method on several mathematical
reasoning benchmarks: AIME, AIME25, AMC, MATH, MINERVA, Olympiad Bench (Olympiad
for short) and GPQA.
All reported performance metrics are averaged over 16 generated responses.

\subsection{The Importance of Off-Policy Data for Optimization}
We first verify the importance of off-policy data reuse across tasks of varying difficulty, demonstrating that frequent data reuse can drastically reduce the number of optimization steps.                                                                       
\begin{figure}[htbp]
    \centering
    \begin{subfigure}{0.32\linewidth}
        \centering
        \includegraphics[width=\linewidth]{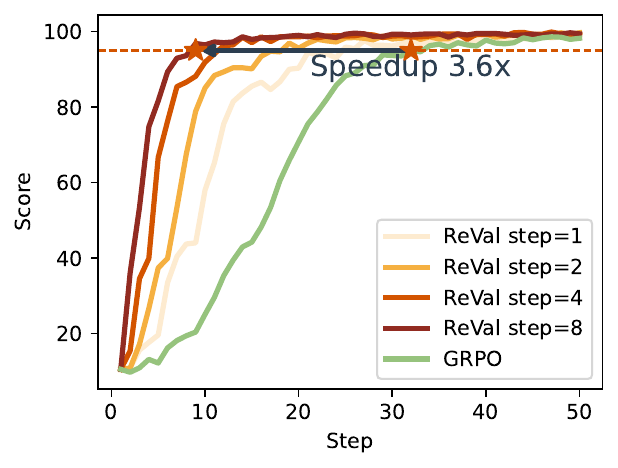}
        \vspace{-16pt}
        \caption{Hard task}
        \label{fig:frequency_hard}
    \end{subfigure}
    \begin{subfigure}{0.32\linewidth}
        \centering
        \includegraphics[width=\linewidth]{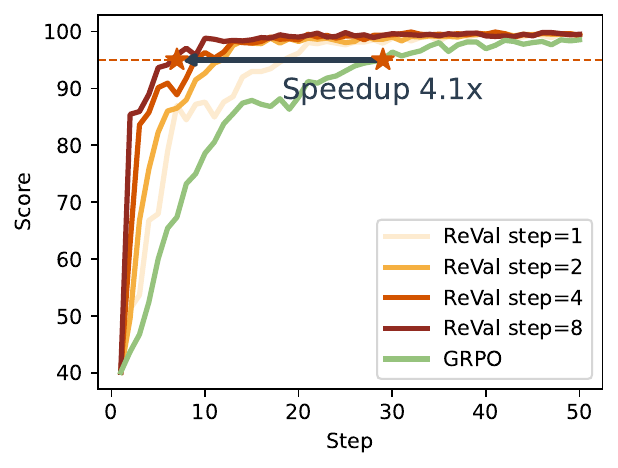}
        \vspace{-16pt}
        \caption{Medium task}
        \label{fig:frequency_medium}
    \end{subfigure}
    \begin{subfigure}{0.32\linewidth}
        \centering
        \includegraphics[width=\linewidth]{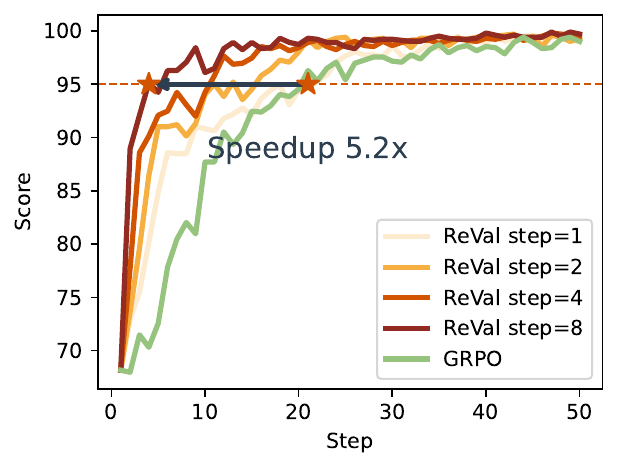}
        \vspace{-16pt}
        \caption{Easy task}
        \label{fig:frequency_easy}
    \end{subfigure}
        \vspace{-3mm}

    \caption{Performance under different data reuse frequencies on tasks with varying difficulty levels.}
    \label{fig:frequency}
\end{figure}
We use the same tasks as in Section~\ref{sec:limitation_of_off}, using a one-shot task learning setting, where training is conducted on a single prompt. We construct three task difficulty levels, quantified by the average success rate at 1024 samples (\texttt{avg@1024}): a hard task with \texttt{avg@1024} = 0.10, a medium task with \texttt{avg@1024} = 0.40, and an easy task with \texttt{avg@1024} = 0.68. We evaluate our proposed \method under different data reuse frequencies (\texttt{step=1}, \texttt{step=2}, \texttt{step=4}, \texttt{step=8}), where \texttt{step=K} indicates that data is sampled from the buffer and optimized in each of $K_{\mathrm{update}}$ before new data collection, and compare against the GRPO which collects fresh data at every optimization step.

As shown in Figure~\ref{fig:frequency}, \method, by reusing data more frequently, achieves an average 4.3× speedup in convergence to high performance across tasks of varying difficulty: for the hard task (Figure~\ref{fig:frequency_hard}), \method with step=9 reaches a score of $95.0$ while GRPO requires 33 steps (3.6x speedup); for the medium task (Figure~\ref{fig:frequency_medium}) and easy task (Figure~\ref{fig:frequency_easy}), \method yield 4.1x and 5.2x speedups. These results confirm that off-policy data reuse directly enhances model performance. Furthermore, the advantage of multiple optimization updates becomes increasingly prominent as task difficulty rises. As observed in the Figure~\ref{fig:frequency}, harder tasks require more optimization steps to reach 95.0. \method achieves the same performance with substantially fewer steps than GRPO, with the gap widening as task difficulty increases. For hard tasks, valid samples are inherently scarce, making intensive and repeated utilization of off-policy data particularly critical for effective policy optimization. These results show the importance of off-policy data for model optimization.


\subsection{Main Results}
\begin{table}[htbp]
    \centering
    \caption{Evaluation performance (\ttt{avg@16}) comparison across different models and benchmarks.}
    \label{tab:main_results}
    \begin{tabular}{l|cccccccc}
    \toprule 
    & {\footnotesize \rotatebox{45}{AIME24}} & {\footnotesize \rotatebox{45}{AIME25}} &{ \footnotesize \rotatebox{45}{AMC}} & { \footnotesize \rotatebox{45}{MATH}} & {\footnotesize \rotatebox{45}{MINERVA} } & {\footnotesize \rotatebox{45}{Olympiad} } & { \footnotesize  \rotatebox{45}{GPQA} } &  {\footnotesize \rotatebox{45}{Avg} } \\
    \midrule
    DPSK-R1-Distill-1.5B & 19.8&  20.0& 50.7& 76.3 & 22.9 & 37.5 & 15.8 & 34.7 \\
    + GRPO & 29.4 & 24.4 & 65.0& 82.6 & 27.6 &  46.3 & 28.8 & 43.4 \\
    + TBRM &26.9 & 22.1   &65.1 & 81.3 & 28.8 & 44.6  &27.3 & 42.3 \\
     \rowcolor{blue!10} + \method & \textbf{32.1} & 23.8 & \textbf{68.6}& \textbf{84.6} & \textbf{30.3} & \textbf{46.6} & \textbf{33.3} & \textbf{45.6} \\
    \hline 
    Qwen2.5-Math-7B  & 19.0 & 6.9& 43.6 & 60.1 & 10.9 & 26.3 & 12.8 & 25.7 \\
    + GRPO & \textbf{34.8} & 12.3& 59.4 & 74.0 & \textbf{30.8} & 37.0 & 20.4 & 38.4 \\
    + TBRM & 30.8  & 10.0 & 58.3 & 74.4 & 27.3 & 37.8 & 19.9 &36.9\\
    \rowcolor{blue!10} + \method & 34.0 &\textbf{13.3} & \textbf{60.2} & \textbf{75.2} & \textbf{30.7} & \textbf{39.2} & \textbf{24.8} & \textbf{39.6}  \\  \bottomrule 
    \end{tabular}%
\end{table}
We conduct experiments on both DPSK-R1-Distill-1.5B and Qwen2.5-Math-7B. The training curves and final evaluation results are presented in Table~\ref{tab:main_results} and Figure~\ref{fig:training_curves_n8}, respectively.
Table~\ref{tab:main_results} reports the final evaluation results, while Figure~\ref{fig:training_curves_n8} presents the evaluation performance curves of different algorithms throughout training. From Table~\ref{tab:main_results}, we observe that \method outperforms the baselines on DPSK-R1-Distill-1.5B on almost all benchmarks, achieving state-of-the-art performance. Beyond the in-domain benchmarks, \method further surpasses GRPO by 4.3\% on the out-of-domain benchmark GPQA, demonstrating the strongest generalization ability. As shown in Figure~\ref{fig:distill_n8}, \method consistently maintains superior performance compared to on-policy methods throughout the entire training process. This trend highlights the importance of incorporating off-policy data, which enables more efficient and stable policy improvement. 

We further conduct the same set of experiments on Qwen2.5-Math-7B. In experiments on Qwen2.5-Math-7B, we observed that the model converges rapidly. We suspect that, for non-reasoning models, outputs tend to be shorter and easier to learn, leading to faster convergence (see Section~\ref{subsec:reset} for detailed discussion). To mitigate this, we introduce reward normalization and periodic reference policy reset, which enable the utilization of negative samples while gradually relaxing the KL constraint. With these enhancements, as shown in Figure~\ref{fig:qwen_n8}, \method achieves performance comparable to or exceeding that of the baseline, with an overall improvement of 1.2\%. On out-of-domain benchmarks, \method similarly achieves a 4.2\% improvement, demonstrating strong generalization performance.
\begin{figure}[htbp]
    \centering

    \begin{subfigure}{\linewidth}
        \centering
        \includegraphics[width=\linewidth]{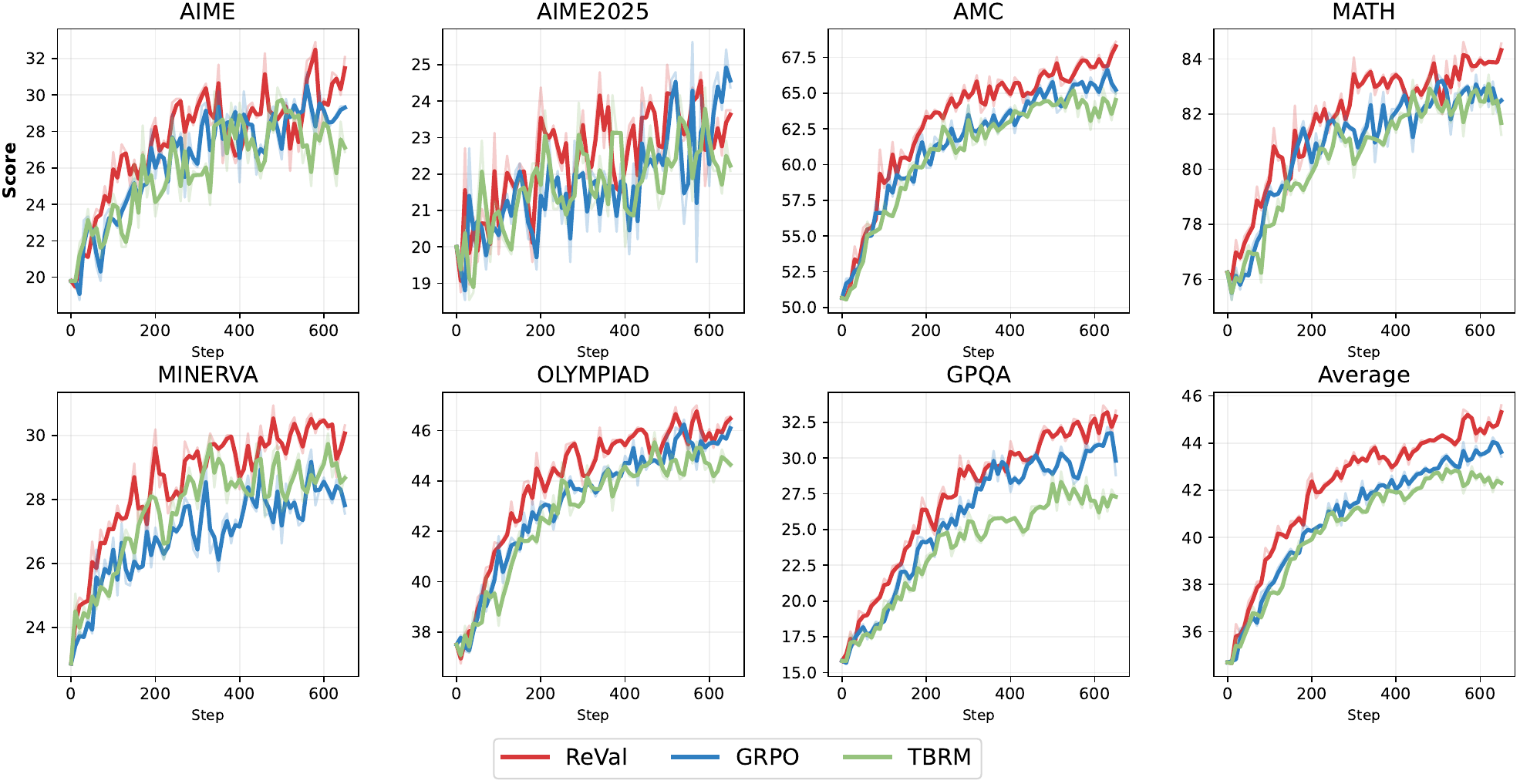}
        \caption{DPSK-R1-Distill-1.5B}
        \label{fig:distill_n8}
    \end{subfigure}
    \begin{subfigure}{\linewidth}
        \centering
        \includegraphics[width=\linewidth]{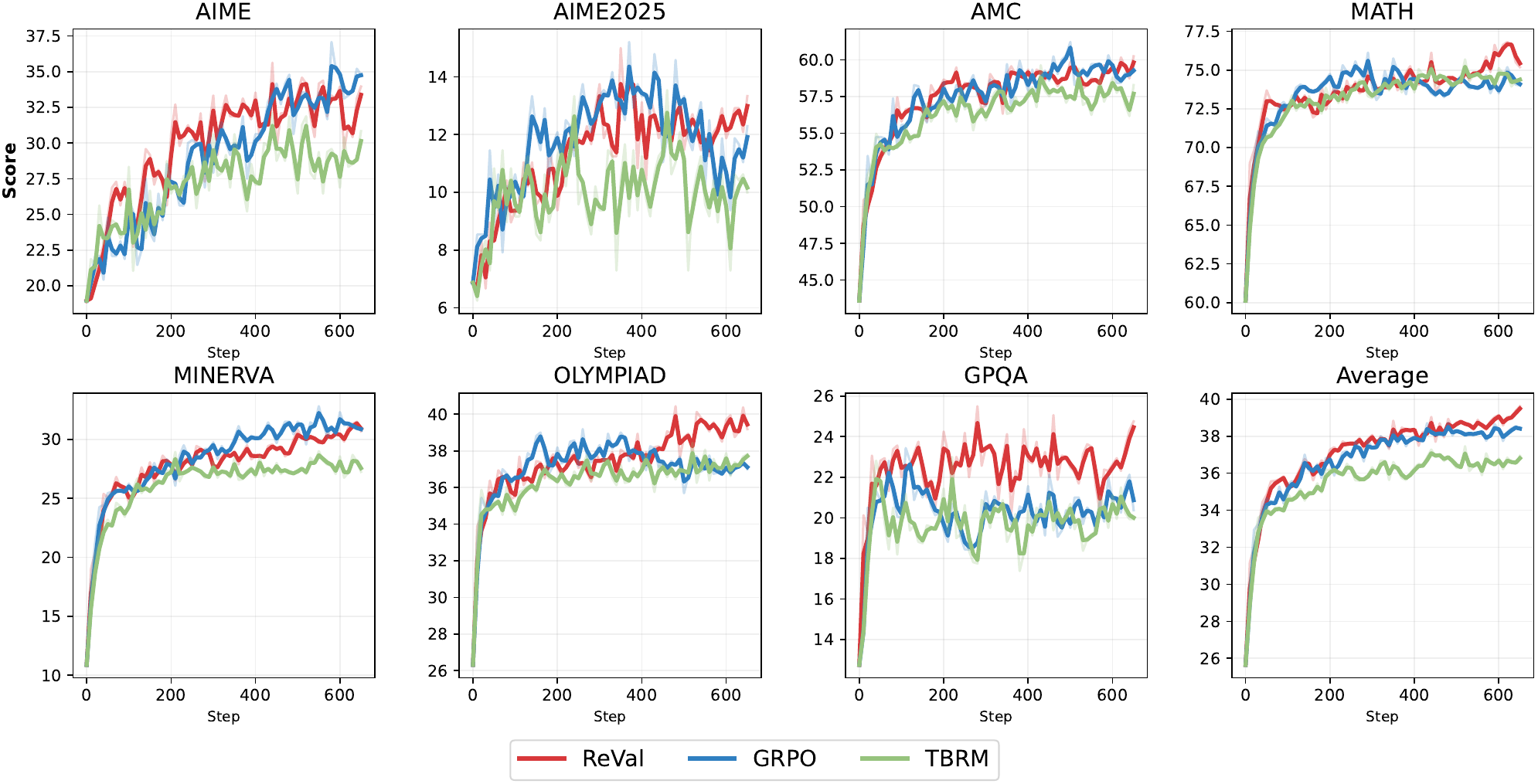}
        \caption{Qwen2.5-Math-7B}
        \label{fig:qwen_n8}
    \end{subfigure}


    \caption{Training curves of DPSK-R1-Distill-1.5B and Qwen2.5-Math-7B. Curves show the accuracy across seven benchmarks (AIME, AIME25, AMC, MATH, Minerva, Olympiad, and GPQA) as well as the average accuracy.}
    \label{fig:training_curves_n8}
\end{figure}

\subsection{Performance under Limited Rollouts}
In many real-world applications, generating on-policy trajectories at every training iteration is prohibitively expensive, as each rollout incurs substantial computational or monetary costs. Under such constraints, improving sample efficiency becomes particularly critical with limited rollouts. To evaluate this aspect, we further investigate whether \method can still achieve strong performance under the extreme setting of $n=1$. We conduct experiments using DPSK-R1-Distill-1.5B with the same set of baselines. 
The training curves are shown in Figure~\ref{fig:training_curves_n1}. On challenging benchmarks such as AIME and GPQA, the advantage of \method is especially pronounced, indicating that off-policy reuse is particularly valuable when rollouts are both expensive and informative. On the Average metric, \method achieves a final score higher than GRPO. These results demonstrate that off-policy learning substantially improves sample utilization efficiency, making \method especially advantageous in low-rollout regimes.
\begin{figure}[h]
    \centering
    \includegraphics[width=\linewidth]{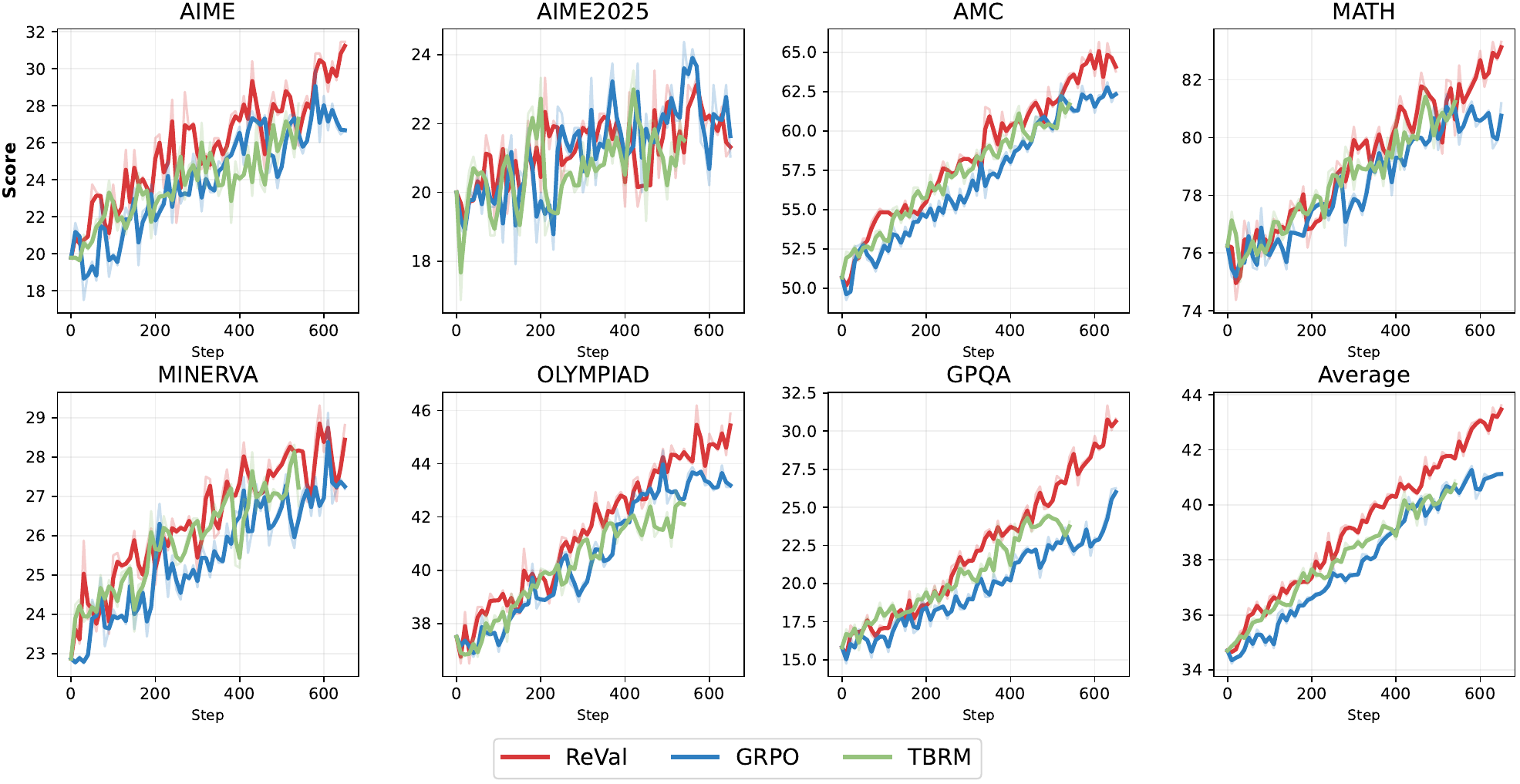}
    \caption{Training Curves of DPSK-R1-Distill-1.5B with N=1}
    \vspace{-2mm}
    \label{fig:training_curves_n1}
\end{figure}

\begin{wrapfigure}{r}{0.45\linewidth}
    \centering
    \includegraphics[width=\linewidth]{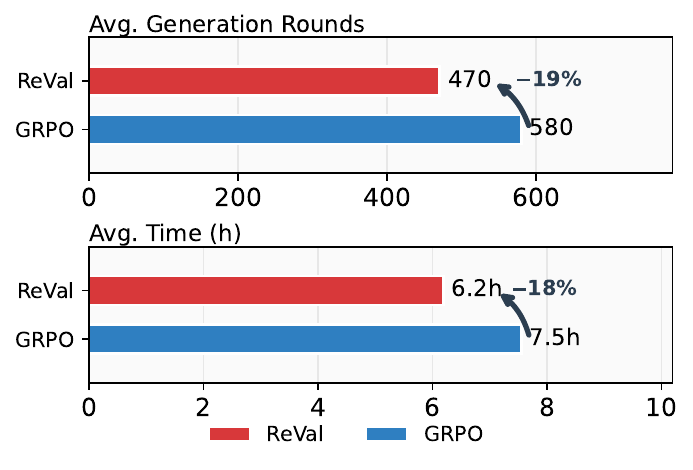}
    \caption{The Average Generation Rounds and Average Total Wall-clock Time (h).}
    \label{fig:bar}
    \vspace{-3mm}
\end{wrapfigure}
Furthermore, under this setting, we report the average number of generation rounds and total wall-clock time required to reach SOTA performance across seven tasks, as shown in Figure~\ref{fig:bar}. \method with more update steps per generation round consistently requires fewer generation rounds, with \method (step=8) reducing the number of generations from 580 (GRPO) to 470. In terms of total training time, \method also achieves the lower time cost at 6.3h, compared to 7.5h for GRPO, a reduction of 1.3h seconds, corresponding to a 18\% decrease in total training time. The experiments show that, when parameter updates are much cheaper than generation (namely, 2.8s per update vs. 36.8s per trajectory), the benefit of off-policy methods lies not only in improved sample efficiency, but also in reduced training cost through more effective reuse of generated trajectories.




\subsection{Key Factors Shaping~\method}
\vspace{2mm}

In this section, we analyze the key factors that influence the performance of \method. We study the effects of the reference policy, the hyperparameter $\beta$, and the utilization of negative samples. All experimental datasets are kept consistent with the main experiments on Qwen2.5-Math-7B, and the final results are reported as the average over seven benchmarks.

\subsubsection{Gradient Dynamics Analysis.}
\vspace{2mm}

To understand the factors that influence training, we analyze the gradient of $\mathcal{L}_{\text{ReVal}}$ with respect to $\theta$:
\begin{equation}
    \nabla_\theta \mathcal{L}_{\text{ReVal}} = -2\,\mathbb{E}_{(x,y)} \left[ \delta(x,y) \cdot \nabla_\theta \log \pi_\theta(y|x) \right],
\end{equation}
where the residual error $\delta(x,y) = \frac{r(x,y)}{\beta} - \left( V_\theta(x) - V_{\text{ref}}(x) + \log \frac{\pi_\theta(y|x)}{\pi_{\text{ref}}(y|x)} \right)$ determines both the magnitude and direction of the gradient update. We identify three key factors that affect the gradient dynamics.

\textbf{KL Regularization and Periodic Reset.} The term $\log \frac{\pi_\theta(y|x)}{\pi_{\text{ref}}(y|x)}$ grows monotonically as the policy diverges from the reference model during training. As this term increases, it progressively reduces $\delta$, weakening the gradient signal and slowing learning. To mitigate this, we periodically reset the reference model to the current policy, which resets the KL term back to zero and restores the magnitude of the gradient signal.

\textbf{Hyperparameter $\beta$.} The parameter $\beta$ directly scales the reward signal via $\frac{r(x,y)}{\beta}$, controlling its relative weight in $\delta$. When $\beta$ is too large, the reward signal is suppressed, causing $\delta$ to remain small throughout training and the gradient to vanish, leading to slow or stalled convergence. Conversely, when $\beta$ is too small, the reward dominates and may cause excessively large gradient updates, destabilizing training.

\textbf{Negative Samples.} When $r(x,y) = 0$, the TD error reduces to:
\begin{equation}
    \delta(x,y) = -\left( V_\theta(x) - V_{\text{ref}}(x) + \log \frac{\pi_\theta(y|x)}{\pi_{\text{ref}}(y|x)} \right),
\end{equation}
and the gradient becomes:
\begin{equation}
    \nabla_\theta \mathcal{L}_{\text{ReVal}} \propto \left( V_\theta(x) - V_{\text{ref}}(x) + \log \frac{\pi_\theta(y|x)}{\pi_{\text{ref}}(y|x)} \right) \cdot \nabla_\theta \log \pi_\theta.
\end{equation}
In this case, the optimization drives $\log \frac{\pi_\theta}{\pi_{\text{ref}}} \to 0$ and $V_\theta-V_{\text{ref}} \to 0$, pulling the policy back toward the reference model rather than decreasing $\log \pi_\theta$. Intuitively, negative samples should penalize incorrect responses by decreasing $\log \pi_\theta(y|x)$, rather than moving to the reference policy.
\vspace{2mm}
\subsubsection{Relaxing KL Regularization via Reference Policy Updates}
\label{subsec:reset}
\vspace{2mm}

\method incorporates KL regularization toward a reference policy, which constrains policy updates but causes the KL term $\log \frac{\pi_\theta}{\pi_{\text{ref}}}$ to grow monotonically during training, progressively weakening the gradient signal. The direct way to mitigate the influence of KL regularization is to periodically reset the reference model to the current policy. By updating $\pi_{\text{ref}}$ in this manner, the effective KL constraint is relaxed, allowing the policy to continue improving without being overly restricted by the accumulated divergence from the initial reference.

In principle, the reference policy could be updated based on the residual error (i.e., the term inside the square in Eq.~\ref{eq:loss}). 
In practice, we find that simple periodic updates work well. We experiment with periodically resetting the reference model every 50/200/400 training steps, as well as a no-reset baseline. 
\begin{wrapfigure}{r}{0.45\linewidth}
    \centering
    \includegraphics[width=\linewidth]{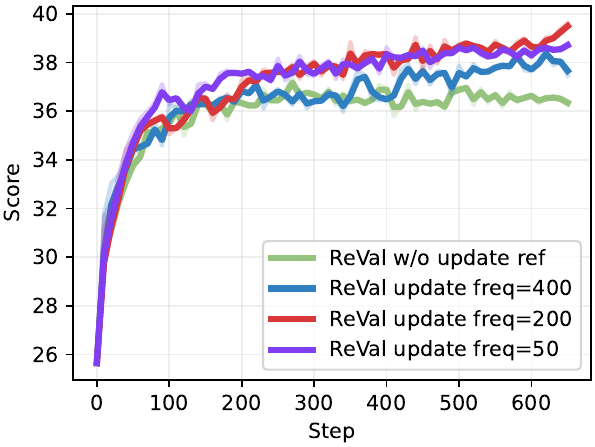}
    \caption{Comparison of different update frequency of reference model.}
    \label{fig:update_ref}
    \vspace{-3mm}
\end{wrapfigure}
We keep all other experimental settings consistent with the main experiments and vary only the update frequency of the reference policy. The results are shown in Figure~\ref{fig:update_ref}. From the Figure~\ref{fig:update_ref}, we can find that without updating the reference policy, the model performance saturates at around 200 training steps and remains unchanged thereafter. Second, updating the reference policy every 200 steps yields the best performance, suggesting that a moderate update frequency provides the most effective balance for training. When the reference policy is updated every 400 steps, a noticeable improvement occurs around the 400th step. These results suggest that periodically resetting the reference policy enables the model to escape the shrinking region induced by reference model, resulting in continued performance improvements during training.

\subsubsection{Hyperparameter \texorpdfstring{$\beta$}{beta}}

The hyperparameter $\beta$ controls the strength of the reward. A larger $\beta$ imposes a weaker signal and keeps the policy closer to the reference model, while a smaller $\beta$ allows more freedom during optimization. The choice of $\beta$ is correlated with the response length. Since the log-ratio term in Eq.~\ref{eq:loss} is summed over tokens, longer responses produce larger accumulated values and thus require a smaller $\beta$ to maintain comparable regularization. In experiments, DPSK-Distill-R1-1.5B generates responses of about 5K tokens and we set $\beta=0.002$, while Qwen2.5-Math-7B produces responses of around 600 tokens and we use $\beta=0.02$.
\begin{figure}[htbp]
    \centering
    \begin{subfigure}{0.46\linewidth}
        \centering
        \includegraphics[width=\linewidth]{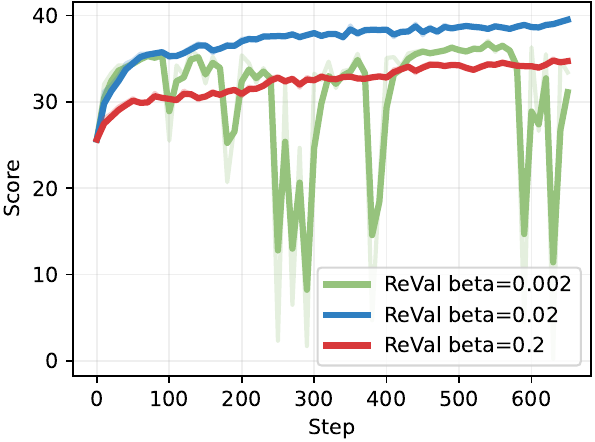}
        \vspace{-16pt}
        \caption{Performance under different $\beta$.}
        \label{fig:difference_beta}
    \end{subfigure}
    \hfill
    \begin{subfigure}{0.46\linewidth}
        \centering
        \includegraphics[width=\linewidth]{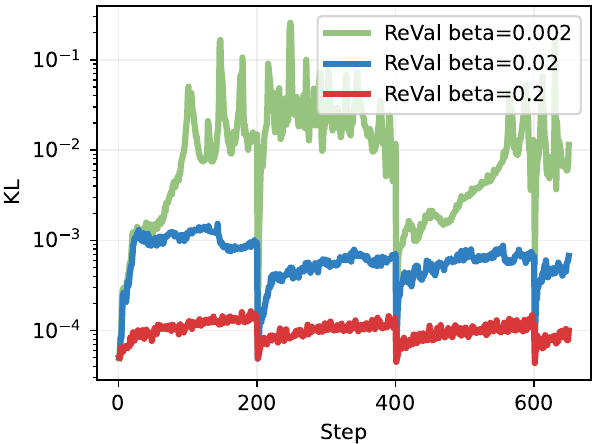}
        \vspace{-16pt}
        \caption{KL divergence under different $\beta$.}
        \label{fig:kl_beta}
    \end{subfigure}
    
    \caption{Effect of the hyperparameter $\beta$. 
    A smaller $\beta$ leads to a larger KL divergence between the policy and the reference model.}
    \vspace{-8pt}
    \label{fig:beta_kl}
\end{figure}
We experiment with different values of $\beta$ (0.2, 0.02, and 0.002), and the results are shown in Figure~\ref{fig:beta_kl}. Specifically, we report the average benchmark performance of the model in Figure~\ref{fig:difference_beta} and the corresponding KL divergence in Figure~\ref{fig:kl_beta}. For KL divergence, the reference policy is updated every 200 training steps, which results in a periodic increase. As shown in Figure~\ref{fig:beta_kl}, the value of $\beta$ significantly affects both the KL divergence and the training dynamics. When $\beta=0.2$, the KL penalty is strong, forcing the policy to stay close to the reference model. In contrast, when $\beta=0.002$, the KL constraint becomes weaker, allowing the policy to deviate further from the reference model. As a result, the KL divergence increases and the policy explores more aggressively. However, overly large deviations can destabilize optimization and eventually lead to training collapse.




\subsubsection{The Utilization of Negative Samples}
\begin{wrapfigure}{r}{0.4\linewidth}
    \centering
    \includegraphics[width=\linewidth]{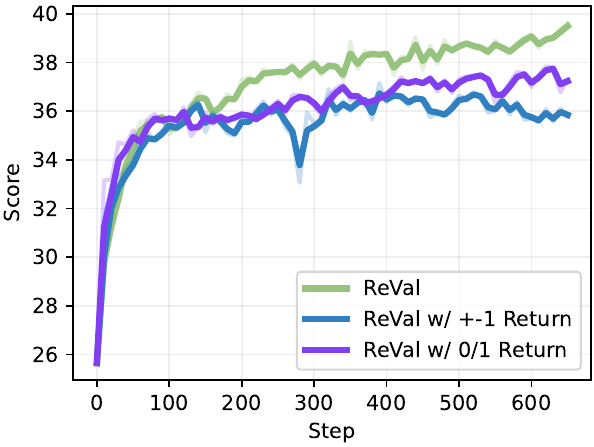}
    \vspace{-15pt}
    \caption{Comparison of different reward designs.}
    \label{fig:difference_reward_design}
\end{wrapfigure}
In reasoning tasks, 0/1 reward is commonly used (i.e., a reward of 1 for correct answers and 0 for incorrect ones). From Equation~\ref{eq:loss}, it can be seen that under this reward scheme the model increases the logits of correct responses. In contrast, when the answer is incorrect, the policy is only nudged toward the reference policy, instead of explicitly suppressing the probability of incorrect samples.
However, negative samples contain informative signals about incorrect behaviors and are crucial for effective learning~\citep{zhu2025surprising}. To better exploit this information, we explore alternative reward formulations that incorporate signals from negative samples. Specifically, we consider two variants. The first variant uses the normalized advantage $\hat{r}_\mathrm{norm}=r(x, y_i)-\operatorname{mean}\left(\left\{r(x, y_i)\right\}_{i=1}^G\right)$ as the reward, following the same formulation as in GRPO. The second variant adopts a $\pm 1$ reward scheme, where correct answers receive a reward of $+1$ and incorrect answers receive a reward of $-1$. As shown in the Figure~\ref{fig:difference_reward_design}, using the normalized advantage yields the best performance, whereas the $\pm 1$ reward scheme can even result in performance degradation.

\section{Related Work}
\subsection{Q-function Representation in LLMs}

\citet{li2025generalist} established a theoretical cornerstone in this direction: they showed that a language model trained via standard next-token prediction implicitly learns a soft Q-function, where the model logits are a principled solution to the Q-function in an offline inverse reinforcement learning formulation. Formally, given a language model $\hat{\pi}$ parameterized as $\hat{\pi}(\cdot \mid s_h) = \text{softmax}(\hat{f}(s_h, \cdot); \alpha)$, the logits $\hat{f}$ directly correspond to the soft Q-values of the data-generating policy, revealing that LLM logits are not arbitrary scores but encode value information about token-level decisions.TBRM~\citep{yuan2025trajectory} also utilized this connection between logits and Q-function, adopting a logit-as-$Q$ parameterization within its Bellman update formulation, and empirically validated the effectiveness of this parameterization. These works suggest that the Q-function in LLMs can be naturally parameterized by the model logits, forming the theoretical foundation of our value-based RL framework.

\subsection{Value-based Reinforcement Learning} 
Efficient RL is important for large-scale LLM training, particularly in asynchronous settings where policy updates and data collection are naturally decoupled~\citep{yan2024efficient, liu2025semantic, ritter2026llms}. Value-based reinforcement learning is a fundamental paradigm for this setting, as it focuses on learning action-value functions and naturally supports off-policy experience reuse. Deep Q-Networks~\citep{dqn_atari} and their numerous variants~\citep{double_dqn} have demonstrated remarkable success in high-dimensional tasks, largely attributed to the effective use of experience replay buffers for stable off-policy learning. To encourage exploration and robustness, a series of works have adopted the Maximum Entropy RL framework~\citep{max_entropy}, where Soft Q-learning (SQL)~\citep{soft-ql} introduces an energy-based formulation to satisfy the entropy-regularized objective. While advanced algorithms such as Soft Actor-Critic~\citep{sac} were further developed and achieved state-of-the-art performance in traditional RL. Inspired by the remarkable performance of value-based methods, we investigate the dynamics of off-policy buffer utilization within the SQL objective, an area that remains relatively under-explored in the context of scaling value-based RL for generative language tasks.

\subsection{Reinforcement Learning in Large Language Models}
RL has become a key component in the post-training stage of LLMs, with reward design and training algorithms being its central elements~\citep{li2025review, panglanguage}. As the field matured, however, it became clear that at LLM scale, an RL algorithm is only practical if it is computationally efficient. ReMax~\citep{li2024remax} was the first to move from actor-critic to actor-only RL, significantly reducing memory usage and training time for LLM post-training. Following ReMax, a series of methods, including GRPO~\citep{shao2024deepseekmath} and DAPO~\citep{qiying2025dapo}, further advanced this low-cost policy optimization paradigm. Policy gradient methods have dominated LLM alignment due to their simplicity, natural compatibility with pretrained language models, and relatively low computational overhead~\citep{li2025review}. But LLM development is increasingly shifting toward \emph{agentic} settings with long and highly variable horizons. The high variance in trajectory lengths~\citep{team2025kimi, fu2025areal}, the training-inference mismatch ~\citep{zhang2026beyond, yao2025offpolicy}, and the difficulty of obtaining sufficient samples~\citep{team2026kimi, gao2025rollpacker} make these methods progressively brittle and inadequate for the demand of off-policy data. These limitations motivate us to turn towards off-policy value-based algorithms. Recently, value-based approaches such as TBRM~\citep{yuan2025trajectory} and ROVER~\citep{he2025randompolicy} have been proposed, which leverage the Q-function information implicitly encoded in the LLM’s own logits for training but they are still trained in on-policy way. Meanwhile, some researchers have begun to train LLMs in an off-policy manner~\citep{zhang2025rlep,zheng2025prosperity}, but the algorithms they employ are still originally designed for on-policy settings. Our method explores combining value-based RL with off-policy training for LLMs.

\section{Conclusion}
In this paper, we investigated the role of off-policy data in LLM RL and proposed \method, a value-based algorithm designed to efficiently leverage historical trajectories. We propose \method, a Bellman-update-based method that combines stepwise signals capturing internal consistency with trajectory-level signals derived from outcome verification. \method naturally supports replay-buffer-based training, allowing efficient reuse of past trajectories. Extensive experiments on standard mathematical reasoning benchmarks demonstrate that \method achieves faster convergence, improves sample efficiency, and outperforms strong baselines such as GRPO, with a \textbf{2.7\%} improvement in AIME24 and \textbf{4.5\%} in out-of-domain benchmark GPQA on the DPSK-R1-Distill-1.5B model. Ablation studies further reveal the impact of key components, including the reference policy, the hyperparameter $\beta$, and different reward and objective designs, providing insights into how each design choice contributes to performance. Currently, our method adopts a standard FIFO replay buffer, which is not the sample-efficient design. In future work, we plan to explore more advanced buffer sampling strategies, such as prioritized experience replay. Furthermore, we will investigate the underlying mechanisms governing the varying update requirements of different data samples in value-based RL for LLM training.

\bibliography{ref}
\clearpage
\appendix
\section{Detailed Experimental Setup}
\label{sec:detailed_exp}
All experiments were implemented using the large-scale reinforcement learning framework Verl (v0.5.0). The default training and inference pipelines were preserved without modification.
For optimization, we adopted a learning rate of 
$1e-6$, following prior recommendations~\citep{qiying2025dapo}. 

For \method, the $\beta$ value was set to 0.002 for DPSK-R1-Distill-1.5B and 0.02 for Qwen-2.5-Math-7B. The buffer size was set to 5,120. For each sampled batch, we perform two updates. First, an on-policy update, followed by an off-policy update using a sample from the replay buffer. For GRPO, the importance sampling clipping thresholds were set asymmetrically to 0.28 (upper) and 0.2 (lower), and a compensation term was applied to account for inconsistencies between vLLM and FSDP. No additional KL or entropy regularization terms were used in any experiments.

DPSK-R1-Distill-1.5B was trained with a maximum sequence length of 8K tokens due to its longer CoT reasoning patterns, which require extended context windows. Qwen2.5-Math-7B was trained with an 8K token limit.

During training, evaluation was conducted every 10 iterations. For each evaluation phase, 16 responses were generated per prompt. To ensure manageable evaluation time, we capped the evaluation set size at 100 samples by randomly sub-sampling benchmarks exceeding this number.

\section{Maximum Entropy Reinforcement Learning}
The original KL-regularized RL objective in \cref{eq:reward_maximization} can be transformed into the following maximum entropy RL objective \citep{sac} with a modified reward.
\begin{align*}
    &\quad \mathbb{E}_{x \sim \rho} \Big[
    \mathbb{E}_{a_{1:H} \sim \pi_{\theta}(\cdot \mid x)}
    \big[
        r_{\text{rule}}(x, a_{1:H})
    \big] - \beta \,
\KL\!\big( 
    \pi_{\theta}(\cdot \mid x),
    \pi_{\text{ref}}(\cdot \mid x)
\big) \Big]
\\
&= \beta \cdot \expect_{\tau \sim \pi} \ls \sum_{h=1}^H \bigg( \underbrace{\frac{r (s_h, a_h)}{\beta} + \log \piref (a_h|s_h)}_{:= r_{\beta} (s_h, a_h)} + \gH (\pi (\cdot|s_h)) \bigg) \rs.
\end{align*}
Here $r (s_h, a_h) = 0$ if $h\not= H$ and $r (s_h, a_h) = r_{\text{rule}} (x, a_{1:H})$ otherwise denotes the token-level reward and $r_{\beta} (s_h, a_h) := r(s_h, a_h) / \beta + \log \piref (a_h|s_h)$ is the modified reward. Besides, $\gH (\pi (\cdot|s_h)) = \expect_{a_h \sim \pi (\cdot|s_h)} [\log (1/\pi(a_h|s_h))]$ denotes the entropy. In maximum entropy RL, the soft optimal Q-function satisfies the Bellman equation.
\begin{align*}
    Q^{\star}_{\beta} (s_h, a_h) = r_{\beta} (s_h, a_h) + \expect_{s_{h+1} \sim P (\cdot|s_h, a_h)} \ls V_{Q^\star_\beta} (s_{h+1})  \rs.
\end{align*}
Here $V_{Q} (s) := \log ( \sum_{a \in \mathcal{A}} \exp (Q (s, a) ) )$ denotes the V-function induced by $Q$. By defining the Bellman operator $(\gT_{\beta} Q) (s_h, a_h) := r_{\beta} (s_h, a_h) + \expect_{s_{h+1} \sim P (\cdot|s_h, a_h)} \ls V_{Q} (s_{h+1})  \rs$, we have that $Q^\star_{\beta}$ is the fixed point w.r.t the Bellman operator, i.e., $Q^\star_{\beta} = \gT_{\beta} Q^\star_{\beta}$. Given the soft optimal Q-function, we can derive the soft optimal policy through a softmax transformation.
\begin{align}
\label{eq:optimal_q_softmax}
    \pi^{\star}_{\beta} (a_h|s_h) = \frac{\exp (Q^{\star}_{\beta} (s_h, a_h))}{\sum_{a \in \gA}\exp (Q^{\star}_{\beta} (s_h, a))} = \exp \lp Q^{\star}_{\beta} (s_h, a_h) - V_{Q^{\star}_{\beta}} (s_h)  \rp.
\end{align}

\section{The Issue of TBRM}
\label{sec:issue_tbrm}

When the reward is zero, the TBRM objective does not reduce to a KL minimization objective. Specifically, setting $r(\tau)=0$ yields
\[
\gL (\theta)=\left(V_\theta(s_1) + \sum_{h=1}^{H} \log \frac{\pi_\theta(a_h \mid s_h)}{\pi_{\text{ref}}(a_h \mid s_h)}\right)^2 .
\]
Minimizing this objective drives the squared term toward zero, which requires the log-likelihood ratio to cancel the value term $V_\theta(s_1)$ rather than directly minimizing the KL divergence between $\pi_\theta$ and $\pi_{\text{ref}}$. As a result, the optimal solution does not correspond to matching the reference policy. To empirically verify this behavior, we additionally conduct experiments where the reward is fixed to $0$. As shown in Figure~\ref{fig:tbrm_zero_reward_kl}, the KL divergence does not converge to zero when the reward is fixed to $0$, empirically confirming that the TBRM objective does not reduce to KL minimization in this case.

\begin{figure}[htbp]
    \centering
    \includegraphics[width=0.6\linewidth]{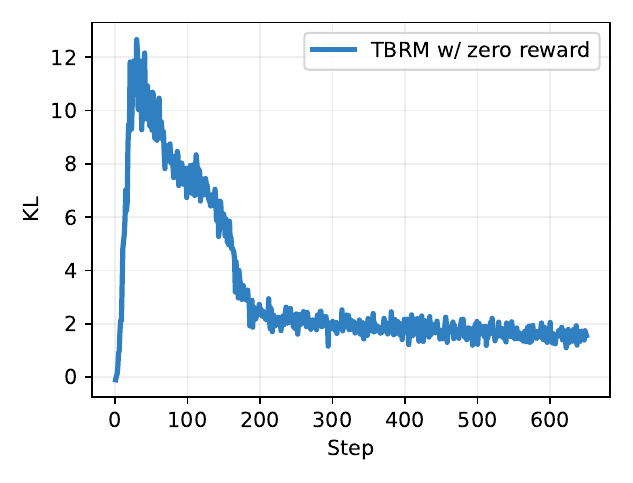}
    \vspace{-12pt}
    \caption{Training behavior of TBRM when the reward is fixed to $0$. The KL divergence between the current policy and the reference policy does not converge to zero.}
    \label{fig:tbrm_zero_reward_kl}
\end{figure}

\section{Proof}
\label{sec:proof}
\begin{proposition}
Consider the objective in Eq.~\ref{eq:loss}. When $r=0$ and the policy is initialized as $\pi_\theta = \pi_{\text{ref}}$, we have $V_\theta(s_1) = V_{\text{ref}}(s_1)$ and $\log \pi_\theta(\tau) = \log \pi_{\text{ref}}(\tau)$, which yields $\mathcal{L}_{\text{ReVal}}(\theta) = 0$.
\end{proposition}
\begin{proof}
When $\pi_\theta = \pi_{\text{ref}}$, the Q-function satisfies $Q_\theta = Q_{\text{ref}}$, which directly implies $V_\theta(s_1) = \log \sum_{a} \exp Q_\theta(s_1, a) = \log \sum_{a} \exp Q_{\text{ref}}(s_1, a) = V_{\text{ref}}(s_1)$. Furthermore, $\log \pi_\theta(\tau) = \sum_{h=1}^H \log \pi_\theta(a_h \mid s_h) = \sum_{h=1}^H \log \pi_{\text{ref}}(a_h \mid s_h) = \log \pi_{\text{ref}}(\tau)$. Substituting into Eq.~\ref{eq:loss} with $r=0$:
\begin{align*}
\mathcal{L}_{\text{ReVal}}(\theta) 
&= \frac{1}{|\mathcal{D}|} \sum_{\tau \in \mathcal{D}} \left( V_\theta(s_1) - V_{\text{ref}}(s_1) + \log \pi_\theta(\tau) - \frac{r_{\text{rule}}(\tau)}{\beta} - \log \pi_{\text{ref}}(\tau) \right)^2 \\
&= \frac{1}{|\mathcal{D}|} \sum_{\tau \in \mathcal{D}} \left( 0 + 0 - 0 \right)^2 = 0.
\end{align*}
\end{proof}

\section{Analysis of Objective Variants}
\label{sec:analysis}

For TBRM, the additional term $V_\theta(s_1)$ in the objective undermines Calibrated Initialization. An alternative variant is to remove $V_\theta(s_1)$, since this term is independent of the action and does not affect the optimal solution. This yields:
\begin{align}
\mathcal{L}_{\text{regression}}(\theta)
&= \frac{1}{|\mathcal{D}|} \sum_{\tau \in \mathcal{D}} \left( \log \pi_\theta(\tau) - \frac{r_{\text{rule}}(\tau)}{\beta} - \log \pi_{\text{ref}}(\tau) \right)^2.
\label{eq:reval_v2}
\end{align}

This objective is closely related to the regression-based formulations in \citet{team2025kimi} and \citet{ritter2026llms}, with the key difference that these work introduce an additional $\log Z(x)$ term. Since $\log Z(x)$ is independent of the policy, it does not affect the optimal solution. In practice, $\log Z(x)$ can be replaced by reward normalization~\citep{team2025kimi}. In regression-based methods, the target labels are fixed, which allows for multiple training passes over the same data stably. 
In Eq.~\ref{eq:reval_v2}, we follow the practice in Kimi K1.5~\citep{team2025kimi} and introduce reward normalization to ensure consistency with the regression-based method:
$\hat{r}_\text{norm} = r_\text{rule}(x, y_i) - \operatorname{mean}\left(\left\{r_\text{rule}(x, y_i)\right\}_{i=1}^G\right)$. However, we empirically find that training remains unstable with reward normalization applied.

\begin{figure}[htbp]
    \centering
    \begin{subfigure}{0.46\linewidth}
        \centering
        \includegraphics[width=\linewidth]{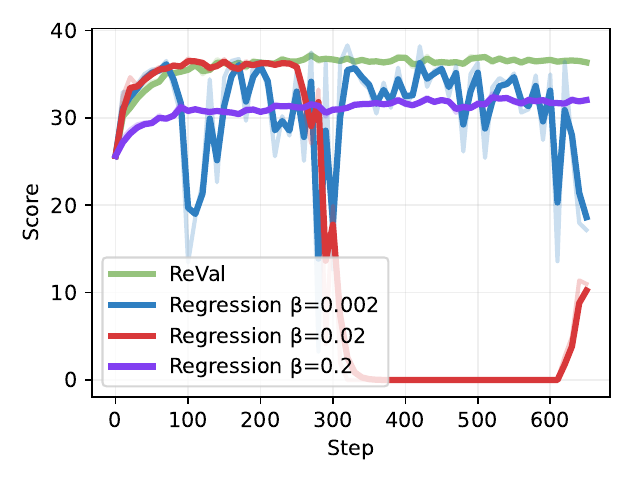}
        \caption{Average benchmark performance.}
        \label{fig:difference_objective}
    \end{subfigure}
    \hfill
    \begin{subfigure}{0.46\linewidth}
        \centering
        \includegraphics[width=\linewidth]{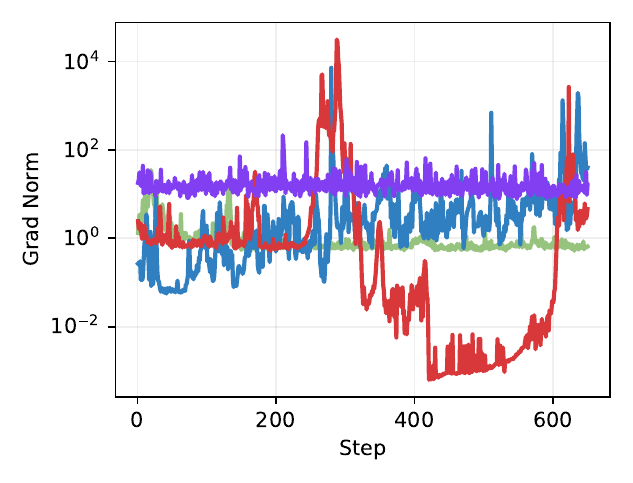}
        \caption{Gradient norm.}
        \label{fig:grad_norm}
    \end{subfigure}
    
\caption{Comparison of different objectives. (a) Average benchmark performance across training steps. (b) Corresponding gradient norm during training.}    
    \label{fig:different_obj}
\end{figure}

We explore different regression by comparing Eq.~\ref{eq:reval_v2} and Eq.~\ref{eq:loss}. 
Specifically, for Eq.~\ref{eq:reval_v2}, we follow the practice in Kimi K1.5~\citep{team2025kimi} and introduce reward normalization:
 i.e., $\hat{r}_\mathrm{norm}=r(x, y_i)-\operatorname{mean}\left(\left\{r(x, y_i)\right\}_{i=1}^G\right)$. However, through experiments, we observe that optimizing the loss in Eq.~\ref{eq:reval_v2} is not stable. We tried different values of $ \beta = 0.2, \ 0.02, \ 0.002 $ and observed that the model remained unstable. We then compared with \method under the same setting (using reward normalization, $ \beta = 0.02 $, and without periodic updates). The results are shown in Figure~\ref{fig:different_obj}. It can be observed that the model remains stable only when $ \beta = 0.2 $, while instability appears at $ \beta = 0.02 $. Under the same conditions, \method exhibits better stability. Experimental observations show that the model can experience extremely large and instable gradient norms (around $ 1\mathrm{e}4 $) which may be related to some anomalous values of $\log \pi / \pi_{\text{REF}} $.

\section{Prompt Templates}\label{exp:prompt}

The prompt templates of all methods used for benchmarking reward models on Multifacted-Bench are shown below. As RM-Bench does not provide specific instructions for each sample, all methods use the default system prompt and the instructions in the \textbf{User} part of the prompt templates will also be removed.

\begin{figure}[htbp]
\fbox{%
\begin{minipage}{0.98\textwidth}
\small

\textbf{User}

\textcolor{blue}{\{Question\}} Let's think step by step and output the final answer within \textbackslash boxed\{\}.

\vspace{0.2cm}

\textbf{Assistant}

<think>

\end{minipage}
}
\caption{Prompt template of DeepSeek-R1-Distill-Qwen-1.5B.}
\label{fig:ours-prompt}
\end{figure}
\begin{figure}[htbp]
\fbox{%
\begin{minipage}{0.98\textwidth}
\small
\textbf{System}

Please reason step by step, and put your final answer within \textbackslash boxed\{\}.

\vspace{0.2cm}

\textbf{User}

\textcolor{blue}{\{Question\}} Let's think step by step and output the final answer within \textbackslash boxed\{\}

\vspace{0.2cm}

\textbf{Assistant}

\end{minipage}
}
\caption{Prompt template of Qwen2.5-Math-7B.}
\end{figure}

\end{document}